\documentclass[12pt,a4paper]{article}

\usepackage[margin=2.5cm]{geometry}
\usepackage{amsmath,amssymb,amsthm}
\usepackage{booktabs}
\usepackage{xcolor}
\usepackage{tikz}
\usepackage{pgfplots}
\pgfplotsset{compat=1.18}
\usepackage{hyperref}
\usepackage{nicefrac}
\usepackage{microtype}
\usepackage{graphicx}
\usepackage{caption}
\usepackage{enumitem}
\usepackage{authblk}

\definecolor{mtfA}{RGB}{219,234,254}
\definecolor{mtfB}{RGB}{147,197,253}
\definecolor{mtfC}{RGB}{59,130,246}
\definecolor{mtfD}{RGB}{29,78,216}
\definecolor{mtfE}{RGB}{30,27,100}

\theoremstyle{plain}
\newtheorem{theorem}{Theorem}[section]
\newtheorem{proposition}[theorem]{Proposition}

\theoremstyle{definition}
\newtheorem{definition}{Definition}[section]
\newtheorem{example}{Example}[section]

\theoremstyle{remark}
\newtheorem{remark}{Remark}[section]

\newcommand{\R}{\mathbb{R}}
\newcommand{\chunk}{\mathrm{chunk}}

\hypersetup{colorlinks=true,linkcolor=blue!60!black,citecolor=blue!60!black}

\title{%
  \textbf{The Temporal Markov Transition Field}\\[0.5em]
  \large A Representation for Time-Varying Transition Dynamics\\[0.3em]
  \large in Time Series Analysis
}
\author{Michael Leznik}
\affil {Aristocrat Leisure Limited, London}
\date{8 March, 2026}

\begin{document}
\maketitle
\thispagestyle{empty}

\begin{abstract}
The Markov Transition Field (MTF), introduced by Wang and Oates (2015), encodes
a time series as a two-dimensional image by mapping each pair of time steps to
the transition probability between their quantile states, estimated from a single
global transition matrix. This construction is efficient when the transition
dynamics are stationary, but produces a misleading representation when the
process changes regime over time: the global matrix averages across regimes and
the resulting image loses all information about \emph{when} each dynamical
regime was active. In this paper we introduce the \emph{Temporal Markov
Transition Field} (TMTF), an extension that partitions the series into $K$
contiguous temporal chunks, estimates a separate local transition matrix for each
chunk, and assembles the image so that each row reflects the dynamics local to
its chunk rather than the global average. The resulting $T \times T$ image has
$K$ horizontal bands of distinct texture, each encoding the transition dynamics
of one temporal segment. We develop the formal definition, establish the key
structural properties of the representation, work through a complete numerical
example that makes the distinction from the global MTF concrete, analyse the
bias--variance trade-off introduced by temporal chunking, and discuss the
geometric interpretation of the local transition matrices in terms of process
properties such as persistence, mean reversion, and trending behaviour. The
TMTF is amplitude-agnostic and order-preserving, making it suitable as an
input channel for convolutional neural networks applied to time series
characterisation tasks.
\end{abstract}
\newpage

\section{Introduction}

A central challenge in applying convolutional neural networks (CNNs) to time
series data is the construction of a two-dimensional representation that
preserves the dynamically relevant structure of the series. Several such
representations have been proposed, including the Gramian Angular Field and
the Recurrence Plot \cite{eckmann1987recurrence}. The Markov Transition
Field \cite{wang2015imaging} is among the most informative for tasks that
depend on persistence or mean-reversion properties, because it encodes the
transition structure of the series at the level of quantile states: rather than
recording what values the series takes, it records how the series moves between
levels.

The global MTF construction, as introduced by Wang and Oates, estimates a single
transition probability matrix from the entire series and uses it to fill a
$T \times T$ image. This single-matrix approach is appropriate when the
transition dynamics are constant over time, in which case pooling all available
transitions produces a low-variance estimate of a single generating mechanism.
However, it has a fundamental limitation: the global matrix cannot distinguish a
series with constant dynamics from one that changes regime midway through. If a
process switches from high persistence to strong mean reversion at time $t^*$,
the global transition matrix averages across the two regimes, and the resulting
image has a uniform texture that resembles neither. The information that the
dynamics \emph{changed}---and \emph{when}---is lost.

The Temporal Markov Transition Field addresses this by replacing the single global
matrix with a set of $K$ local matrices, one per temporal segment. The entry at
position $(i,j)$ of the image is filled using the local matrix estimated from the
segment containing time step $i$, rather than the global matrix. The image
acquires horizontal band structure: all rows within the same segment share the
same row-texture (governed by the same local matrix), while rows from different
segments carry the texture of their respective local dynamics. A CNN reading this
image can detect the bands and learn to associate their patterns---diagonal-heavy
for persistent segments, spread for mean-reverting ones, upper-triangular for
trending ones---with the dynamical properties of the generating process.

This paper is organised as follows. Section~2 establishes the building blocks:
quantile binning and the empirical transition matrix. Section~3 recalls the global
MTF and identifies its limitation precisely through a worked example.
Section~4 introduces the TMTF formally, establishes its key properties, and
analyses the bias--variance trade-off. Section~5 continues the worked example
through the TMTF construction. Section~6 interprets the local transition matrices
in terms of process properties. Section~7 discusses the multi-resolution extension.
Section~8 concludes.

\section{Building Blocks}

Let $\mathbf{x} = (x_1, x_2, \ldots, x_T) \in \R^T$ denote a time series of
length $T$, observed at equal time intervals. The MTF family of representations
operates on the \emph{relative ordering} of the observations rather than their
absolute values, achieved through quantile binning.

\subsection{Quantile Binning and the State Sequence}

\begin{definition}[Quantile Bins and State Sequence]
\label{def:bins}
Given a time series $\mathbf{x} \in \R^T$ and a bin count $Q \in \mathbb{N}$,
let $q_0 < q_1 < \cdots < q_Q$ be the empirical quantile boundaries of
$\mathbf{x}$ such that each interval $[q_{k-1}, q_k)$ contains
$\lfloor T/Q \rfloor$ observations (with boundary adjustments for divisibility).
The \emph{quantile state} of observation $x_t$ is
\[
  b_t = k \quad \text{if} \quad x_t \in [q_{k-1},\, q_k),
  \qquad k = 1, \ldots, Q.
\]
The sequence $\mathbf{b} = (b_1, b_2, \ldots, b_T)$ is the \emph{state sequence}
of $\mathbf{x}$.
\end{definition}

The state sequence discards the magnitudes of the observations and retains only
their rank structure. Two time series related by any monotone transformation
produce the same state sequence. The MTF and TMTF are therefore
\emph{amplitude-agnostic}: invariant to rescaling, shifting, and any other
transformation that preserves relative ordering.

The choice of $Q$ controls the resolution of the representation. A large $Q$
produces fine bins and can capture subtle distributional features of the
transitions, but requires more data per bin to estimate transition probabilities
reliably. A small $Q$ is robust with fewer observations but coarser. In practice,
$Q \in \{6, 10, 14\}$ balance resolution and reliability for series of length
$T \ge 200$.

\subsection{The Empirical Transition Matrix}

\begin{definition}[Empirical Transition Probability Matrix]
\label{def:trans-matrix}
Given a state sequence $\mathbf{b} = (b_1, \ldots, b_T)$, the $Q \times Q$
\emph{empirical transition probability matrix} $W$ has entries
\[
  W_{kl}
  = \frac{\#\{t \in \{1,\ldots,T-1\} : b_t = k \;\text{and}\; b_{t+1} = l\}}
         {\#\{t \in \{1,\ldots,T-1\} : b_t = k\}},
  \quad k, l \in \{1,\ldots,Q\}.
\]
Each row sums to unity: $\sum_{l=1}^{Q} W_{kl} = 1$ for all $k$.
\end{definition}

The entry $W_{kl}$ estimates the probability that a series in quantile state $k$
at one time step will be in quantile state $l$ at the next step. The diagonal
entries $W_{kk}$ are particularly informative: large $W_{kk}$ indicates that the
series tends to remain in state $k$ (persistence); small $W_{kk}$ indicates rapid
departure (mean reversion).

Only $T-1$ transitions are available from a series of length $T$, because the
final observation $x_T$ has no successor. The row total for state $k$ therefore
equals the count of times the series occupies state $k$ among
$x_1, \ldots, x_{T-1}$, equivalently the total occurrence count minus one if the
series occupies state $k$ at the final time step.

\section{The Global Markov Transition Field}

\subsection{Definition and Row Degeneracy}

The global MTF \cite{wang2015imaging} assembles the $T \times T$ image by
assigning to each cell $(i,j)$ the transition probability from the state at time
$i$ to the state at time $j$, as estimated by the single global matrix $W$.

\begin{definition}[Global Markov Transition Field]
\label{def:global-mtf}
Given a time series $\mathbf{x}$ with state sequence $\mathbf{b}$ and global
transition matrix $W$, the \emph{global MTF} is the $T \times T$ matrix $M$
with entries
\[
  M_{ij} = W_{b_i,\, b_j}, \qquad i, j \in \{1,\ldots,T\}.
\]
\end{definition}

The entry $M_{ij}$ answers: how probable is the transition from the level at time
$i$ to the level at time $j$, as estimated from the entire series? Because $W$ is
global, every time step contributes equally regardless of when it occurred. A
fundamental structural consequence follows.

\begin{proposition}[Row Degeneracy of the Global MTF]
\label{prop:row-degen}
Two rows $i$ and $i'$ of the global MTF are identical if and only if
$b_i = b_{i'}$.
\end{proposition}

\begin{proof}
$M_{ij} = W_{b_i, b_j}$ for all $j$. If $b_i = b_{i'}$ then $M_{ij} =
W_{b_i, b_j} = W_{b_{i'}, b_j} = M_{i'j}$ for every $j$. Conversely, if
$b_i \ne b_{i'}$ then rows $b_i$ and $b_{i'}$ of $W$ differ in general.
\end{proof}

Proposition~\ref{prop:row-degen} exposes the fundamental limitation. The global
MTF has at most $Q$ distinct row patterns---one per quantile state. All time steps
in the same state produce identical rows regardless of when they occur. The image
contains no information about the temporal location of dynamical regimes: it
cannot distinguish a series with constant dynamics from one whose mechanism
changed at some interior time point.

\subsection{Worked Example: Global MTF on a Regime-Switching Series}

\begin{example}[Global MTF]
\label{ex:global}
Consider
\[
  \mathbf{x} = (12,\; 85,\; 45,\; 18,\; 78,\; 42,\; 15,\; 22,\; 55,\; 48,\; 82,\; 91).
\]
The first six observations zigzag between low and high values (mean-reverting
dynamics). The last six rise monotonically and stay high (persistent dynamics).

\begin{figure}[htbp]
\centering
\begin{tikzpicture}
\begin{axis}[
  width=13cm, height=4.2cm,
  xlabel={Time $t$}, ylabel={$x_t$},
  xtick={1,2,...,12}, xmin=0.5, xmax=12.5, ymin=-5, ymax=105,
  grid=both, grid style={gray!20},
  title={Example series with $Q=3$ quantile bands and $K=2$ chunk boundary},
  title style={font=\small},
  tick label style={font=\footnotesize}, label style={font=\small},
]
\fill[red!8]   (axis cs:0.5,-5)  rectangle (axis cs:12.5,35);
\fill[green!8] (axis cs:0.5,35)  rectangle (axis cs:12.5,63);
\fill[blue!8]  (axis cs:0.5,63)  rectangle (axis cs:12.5,105);
\node[font=\tiny,red!70!black,anchor=west]   at (axis cs:0.6,18) {Low ($k=1$)};
\node[font=\tiny,green!60!black,anchor=west] at (axis cs:0.6,49) {Mid ($k=2$)};
\node[font=\tiny,blue!70!black,anchor=west]  at (axis cs:0.6,83) {High ($k=3$)};
\addplot[blue!70!black,thick,mark=*,mark size=2pt]
  coordinates{(1,12)(2,85)(3,45)(4,18)(5,78)(6,42)
              (7,15)(8,22)(9,55)(10,48)(11,82)(12,91)};
\draw[red!70!black,dashed,thick] (axis cs:6.5,-5)--(axis cs:6.5,105)
  node[pos=1,above,font=\tiny,red!70!black]{chunk boundary};
\end{axis}
\end{tikzpicture}
\caption{The 12-point example series. Quantile bands are shaded. The dashed
vertical line marks the chunk boundary for $K=2$.}
\label{fig:series}
\end{figure}
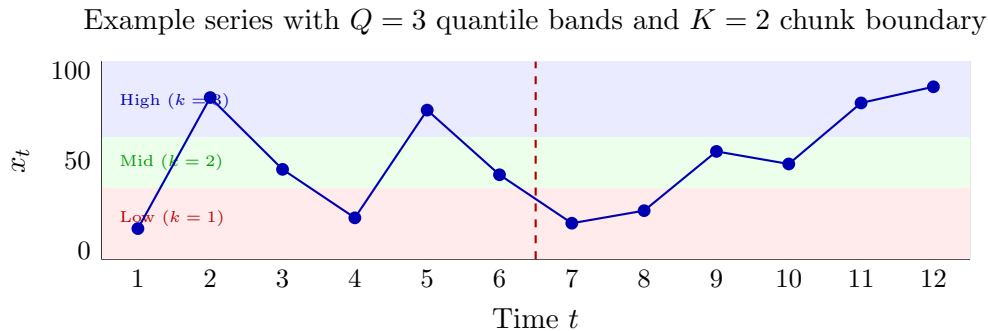

\textbf{Quantile binning} ($Q=3$, four observations per bin):
\begin{center}
\begin{tabular}{llcc}
\toprule
State & Range & Values & Times \\
\midrule
$k=1$ (Low)  & $x \le 35$      & $12,18,15,22$ & $1,4,7,8$ \\
$k=2$ (Mid)  & $35<x\le 63$    & $45,42,55,48$ & $3,6,9,10$ \\
$k=3$ (High) & $x>63$          & $85,78,82,91$ & $2,5,11,12$ \\
\bottomrule
\end{tabular}
\end{center}
State sequence: $\mathbf{b} = (1,3,2,1,3,2,1,1,2,2,3,3)$.

\textbf{Transition tally} (11 consecutive pairs):
\begin{center}
\begin{tabular}{lrrrr}
\toprule
From & To 1 & To 2 & To 3 & Total \\
\midrule
State 1 & 1 & 1 & 2 & 4 \\
State 2 & 2 & 1 & 1 & 4 \\
State 3 & 0 & 2 & 1 & 3 \\
\bottomrule
\end{tabular}
\end{center}
Row totals sum to $T-1=11$.\; State 3 yields only 3 transitions because the
series occupies state 3 at $t=12$ (the final step), which has no successor.
\[
  W = \begin{pmatrix}
    0.25 & 0.25 & 0.50 \\
    0.50 & 0.25 & 0.25 \\
    0.00 & 0.67 & 0.33
  \end{pmatrix}.
\]
The diagonal entries are modest (0.25, 0.25, 0.33), reflecting neither pure
mean-reversion nor pure persistence. The global matrix is a blend of both
regimes. By Proposition~\ref{prop:row-degen}, the $12\times 12$ image has
exactly three distinct row patterns. Rows at $t\in\{1,4,7,8\}$ (all state 1)
are identical; rows at $t=7$ and $t=1$---from opposite halves of the
series---cannot be distinguished. The regime change is invisible.

\begin{figure}[htbp]
\centering
\includegraphics[width=0.82\textwidth]{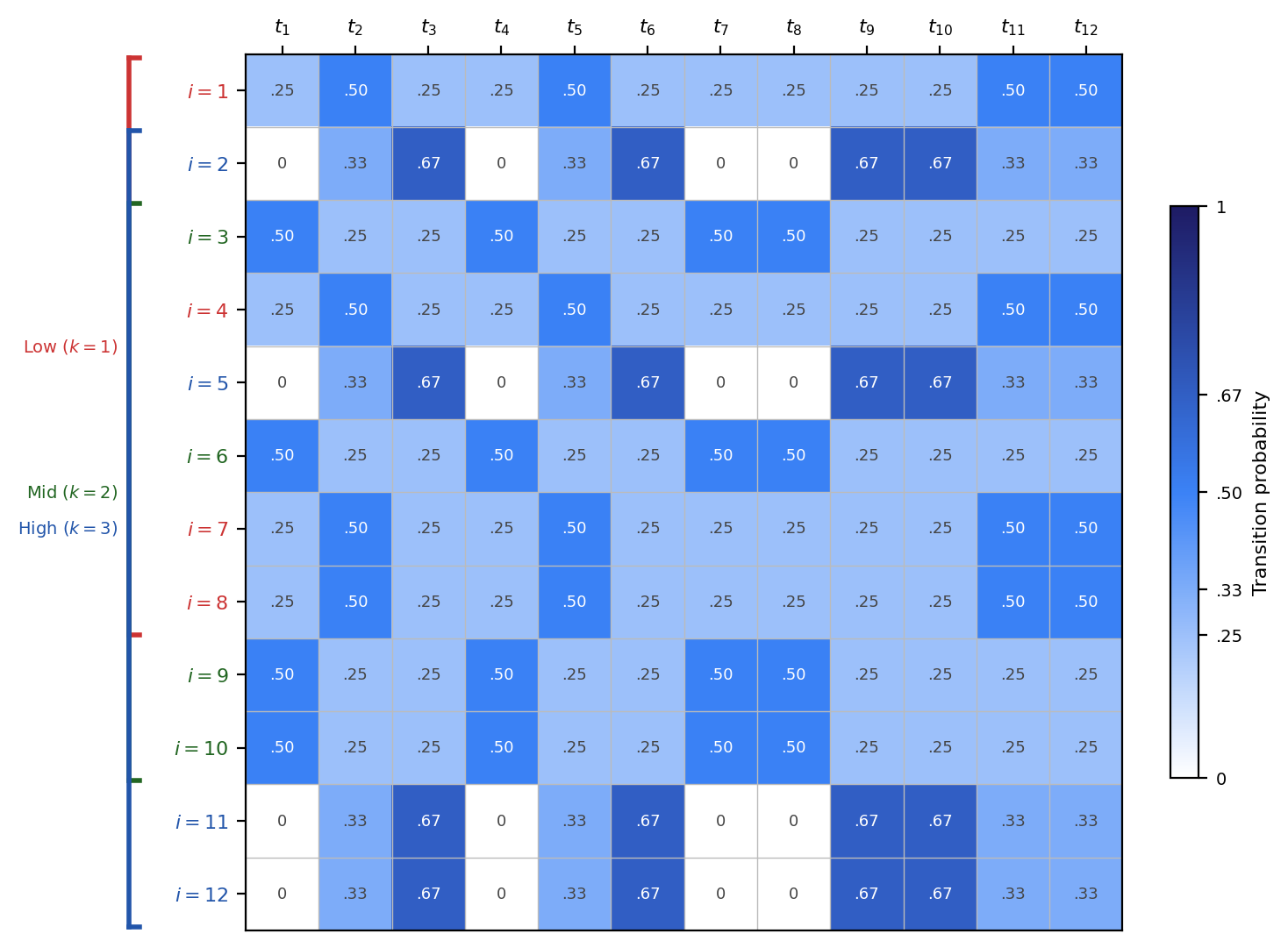}
\caption{The assembled $12\times12$ global MTF image for the worked example
($Q=3$). Row-index colours and left-side brackets group the three sets of
identical rows: red for state~1 ($i \in \{1,4,7,8\}$), green for state~2
($i \in \{3,6,9,10\}$), and blue for state~3 ($i \in \{2,5,11,12\}$). Within
each group the rows are pixel-for-pixel identical regardless of whether they
fall in the mean-reverting first half or the persistent second half of the
series. No band structure is visible; the regime transition at $t=6$ leaves
no trace.}
\label{fig:global-mtf-image}
\end{figure}
\end{example}

\section{The Temporal Markov Transition Field}
\label{sec:tmtf}

\subsection{Formal Definition}

\begin{definition}[Temporal Segmentation]
\label{def:chunks}
Given a time series of length $T$ and a chunk count $K \in \mathbb{N}$ with
$K \mid T$, the $k$-th \emph{temporal chunk} is the contiguous index set
\[
  C_k = \bigl\{t \in \mathbb{N} : (k-1)\tfrac{T}{K} < t \le k\tfrac{T}{K}\bigr\},
  \qquad k = 1, \ldots, K.
\]
The \emph{chunk function} $\chunk : \{1,\ldots,T\} \to \{1,\ldots,K\}$ maps each
time index to its chunk: $\chunk(t) = k$ iff $t \in C_k$.
\end{definition}

\begin{definition}[Local Transition Matrix]
\label{def:local-matrix}
For chunk $C_k$, the \emph{local transition matrix} $W^{(k)}$ is the
$Q \times Q$ empirical transition probability matrix estimated from the
transitions within $C_k$ only:
\[
  W^{(k)}_{lm}
  = \frac{\#\{t \in C_k : b_t = l,\; b_{t+1} = m,\; t+1 \in C_k\}}
         {\#\{t \in C_k : b_t = l,\; t+1 \in C_k\}},
  \quad l, m \in \{1,\ldots,Q\}.
\]
Only within-chunk consecutive pairs are counted; transitions crossing a chunk
boundary are excluded.
\end{definition}

\begin{definition}[Temporal Markov Transition Field]
\label{def:tmtf}
Given state sequence $\mathbf{b}$, chunk count $K$, and local matrices
$W^{(1)}, \ldots, W^{(K)}$, the \emph{Temporal Markov Transition Field} is the
$T \times T$ matrix $\mathcal{M}$ with entries
\begin{equation}
  \mathcal{M}_{ij}
  = W^{(\chunk(i))}_{b_i,\, b_j},
  \qquad i, j \in \{1,\ldots,T\}.
  \label{eq:tmtf}
\end{equation}
\end{definition}

The entry $\mathcal{M}_{ij}$ answers: given the dynamics prevailing in the
temporal segment containing time step $i$, how probable is the transition from
the state at time $i$ to the state at time $j$? The probability is drawn from
the local matrix $W^{(\chunk(i))}$ estimated from $C_{\chunk(i)}$ alone.

\subsection{Structural Properties}

\begin{proposition}[Band Structure]
\label{prop:bands}
The TMTF image is partitioned into $K$ horizontal bands corresponding to
$C_1, \ldots, C_K$. Rows $i, i'$ in the same chunk produce the same row
texture if and only if $b_i = b_{i'}$. Rows $i, i'$ in \emph{different}
chunks may produce distinct textures even when $b_i = b_{i'}$, provided
$W^{(\chunk(i))} \ne W^{(\chunk(i'))}$.
\end{proposition}

\begin{proof}
Within a chunk, $\mathcal{M}_{ij} = W^{(\chunk(i))}_{b_i,b_j}$. For
$\chunk(i)=\chunk(i')$, the same local matrix governs both rows, so the rows
agree iff $b_i = b_{i'}$ by an argument identical to
Proposition~\ref{prop:row-degen}. Across chunks, if
$W^{(c)} \ne W^{(c')}$ for $c=\chunk(i)$, $c'=\chunk(i')$, then even with
$b_i=b_{i'}$ there exists $j$ with $W^{(c)}_{b_i,b_j} \ne W^{(c')}_{b_i,b_j}$.
\end{proof}

The TMTF can therefore produce up to $K \cdot Q$ distinct row patterns, compared
to at most $Q$ for the global MTF. The additional $K$-fold discrimination is
precisely the temporal information that the global MTF discards.

\begin{proposition}[Graceful Degradation to Global MTF]
\label{prop:degrade}
If $W^{(1)} = \cdots = W^{(K)} = W$, then $\mathcal{M}_{ij} = M_{ij}$ for
all $i,j$: the TMTF reduces to the global MTF.
\end{proposition}

\begin{proof}
$\mathcal{M}_{ij} = W^{(\chunk(i))}_{b_i,b_j} = W_{b_i,b_j} = M_{ij}$.
\end{proof}

Proposition~\ref{prop:degrade} ensures that the TMTF does not misrepresent
genuinely stationary series: when the dynamics do not vary across chunks, the
image is identical to the global MTF. The gain in representational flexibility
comes at no cost when that flexibility is not needed.

\begin{proposition}[Amplitude Agnosticism]
\label{prop:amplitude}
The TMTF is invariant to any strictly increasing transformation
$f : \R \to \R$ applied to the observations.
\end{proposition}

\begin{proof}
A strictly increasing $f$ preserves the rank ordering of observations and hence
the quantile boundaries and state sequence $\mathbf{b}$. All local matrices
$W^{(k)}$ are functions of $\mathbf{b}$ alone, so they are unchanged. By
\eqref{eq:tmtf}, $\mathcal{M}_{ij}$ is unchanged.
\end{proof}

\subsection{The Column Asymmetry}

Equation~\eqref{eq:tmtf} applies the chunk function to the \emph{row} index $i$
but not to the column index $j$. The column state $b_j$ is drawn from the global
state sequence, spanning the entire series. This asymmetry is deliberate.

The entry $\mathcal{M}_{ij}$ concerns the transition \emph{departing} from time
$i$. The dynamics relevant to this departure are those prevailing around time $i$,
encoded in $W^{(\chunk(i))}$. The destination at time $j$ identifies only the
target state; the dynamics at time $j$ do not govern the departure from $i$.

A concrete consequence: the columns are not partitioned into vertical bands. The
column $j$ appearing in a row from $C_1$ and the same column $j$ appearing in a
row from $C_2$ both refer to the same destination state $b_j$ at the same time
$j$. The two rows are therefore directly comparable: if they differ, it is because
the local transition matrices differ, not because the destination states differ.
This cross-chunk comparison information would be lost if the image were decomposed
into $K$ independent block matrices.

\subsection{Bias--Variance Trade-off}

The local matrix $W^{(k)}$ is estimated from $T/K - 1$ transitions, compared to
$T-1$ for the global matrix. With fewer observations, $W^{(k)}$ has higher
variance. The expected squared error of each entry $\hat{W}^{(k)}_{lm}$ is of
order $1/(n_k \cdot \hat{P}(b_t=l))$ where $n_k = T/K - 1$, increasing as $K$
grows.

Against this variance cost is a bias reduction when the dynamics are genuinely
non-stationary. If the true local matrix in chunk $k$ is $W^{(k)}_* \ne W_*$
(the global true matrix), the global estimator carries bias $W - W^{(k)}_*$ for
entries in chunk $k$, while the local estimator is unbiased within the chunk.

A practical minimum-transitions rule requires each local matrix to have at least
$5Q$ transitions per row on average, implying $n_k \ge 5Q^2$, i.e.,
$T/K \ge 5Q^2 + 1$. For $Q=6$ and $T=400$ this permits $K \le 2.2$; for $T=1000$
it permits $K \le 5.5$. For the typical range $T \in [200,1000]$ with
$Q \in \{6,10,14\}$, $K=4$ is a robust default that provides temporal resolution
while keeping each chunk above the estimation threshold at $T \ge 400$.

\section{Worked Example: TMTF with \texorpdfstring{$K=2$}{K=2}}

\begin{example}[TMTF with $K=2$]
\label{ex:tmtf}
Continuing from Example~\ref{ex:global}, we apply the TMTF with $K=2$ equal
chunks of six observations each.

\textbf{Temporal segments.}
$C_1 = \{1,\ldots,6\}$ with $\mathbf{b}^{(1)} = (1,3,2,1,3,2)$;
$C_2 = \{7,\ldots,12\}$ with $\mathbf{b}^{(2)} = (1,1,2,2,3,3)$.
Both chunks visit all three states. The regime difference is entirely in how the
series moves between them.

\textbf{Local matrix $W^{(1)}$ (mean-reverting regime).}
Within-chunk transitions (transition from $t=6$ to $t=7$ is excluded as the
successor falls in $C_2$):
$1\to3,\quad 3\to2,\quad 2\to1,\quad 1\to3,\quad 3\to2$.

\begin{center}
\begin{tabular}{lrrrr}
\toprule
From & To 1 & To 2 & To 3 & Total \\
\midrule
State 1 & 0 & 0 & 2 & 2 \\
State 2 & 2 & 0 & 0 & 2 \\
State 3 & 0 & 2 & 0 & 2 \\
\bottomrule
\end{tabular}
\end{center}

Row totals sum to $|C_1| - 1 = 5$.\checkmark
\[
  W^{(1)} = \begin{pmatrix}
    0 & 0 & 1 \\
    1 & 0 & 0 \\
    0 & 1 & 0
  \end{pmatrix}.
\]
Every diagonal entry is zero: the series never stays in the same state. The
unit off-diagonal entries encode the deterministic cycle
$1\to3\to2\to1\to\cdots$. This is an extreme mean-reverting signature.

\textbf{Local matrix $W^{(2)}$ (persistent regime).}
Within-chunk transitions (transition from $t=12$ excluded):
$1\to1,\quad 1\to2,\quad 2\to2,\quad 2\to3,\quad 3\to3$.

\begin{center}
\begin{tabular}{lrrrr}
\toprule
From & To 1 & To 2 & To 3 & Total \\
\midrule
State 1 & 1 & 1 & 0 & 2 \\
State 2 & 0 & 1 & 1 & 2 \\
State 3 & 0 & 0 & 1 & 1 \\
\bottomrule
\end{tabular}
\end{center}

\begin{remark}
State 3 appears twice in $C_2$ (at $t=11$ and $t=12$) but contributes only
one outgoing transition, since $t=12$ is the final observation. Row totals
sum to $|C_2| - 1 = 5$.\checkmark
\end{remark}

\[
  W^{(2)} = \begin{pmatrix}
    0.50 & 0.50 & 0.00 \\
    0.00 & 0.50 & 0.50 \\
    0.00 & 0.00 & 1.00
  \end{pmatrix}.
\]
The diagonal entries are large (0.50, 0.50, 1.00): the series strongly tends to
stay in its current state. The matrix is upper-triangular with zero
lower-triangular entries: the series can stay or move to a higher state, but
never returns to a lower state. State 3 is absorbing ($W^{(2)}_{33}=1.0$):
this is the signature of a persistent upward trend.

\begin{figure}[htbp]
\centering
\includegraphics[width=0.82\textwidth]{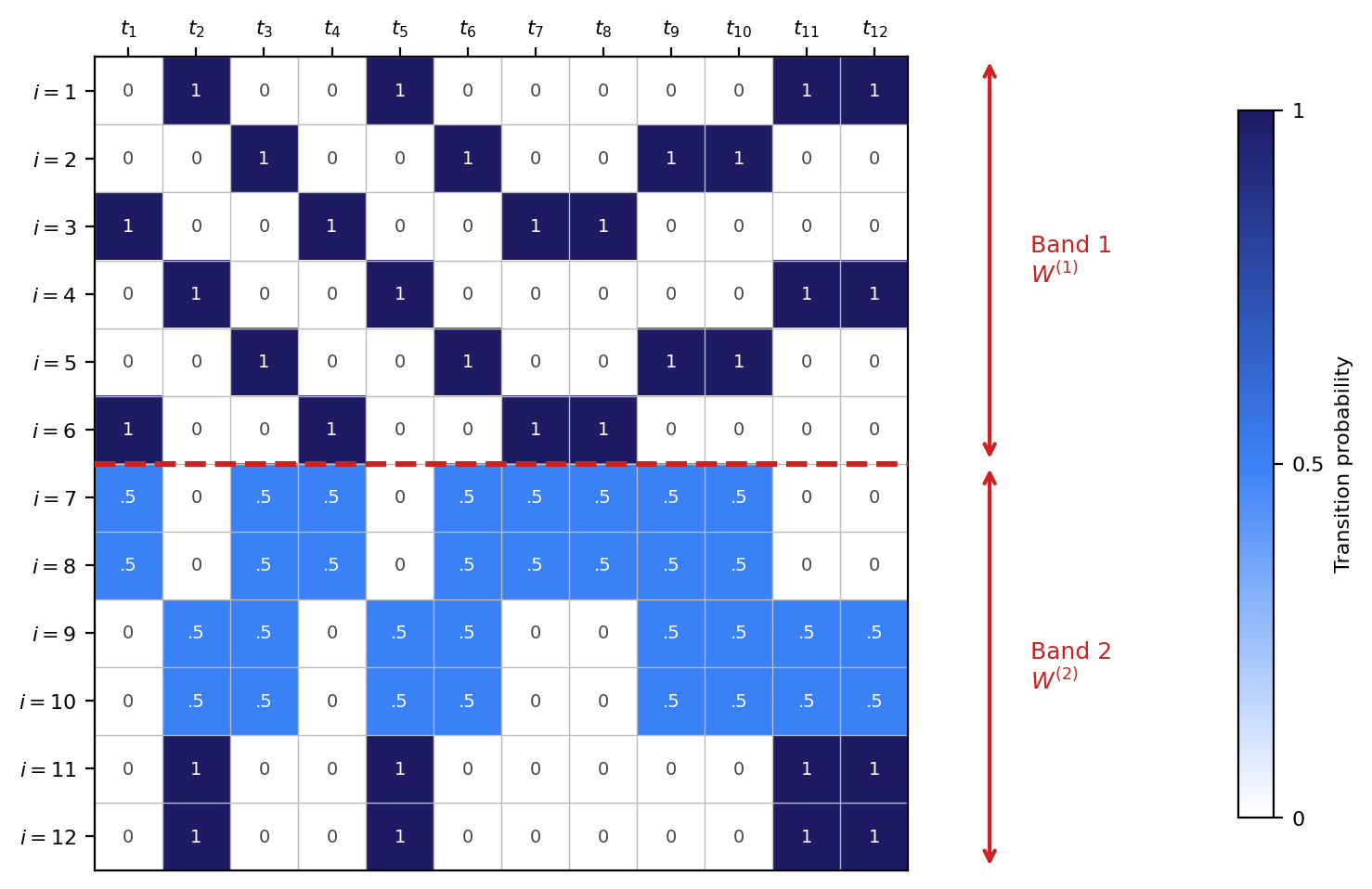}
\caption{The assembled $12\times12$ TMTF image for the worked example ($K=2$,
$Q=3$). Rows $i=1$--$6$ (top band, above the dashed line) are governed by
$W^{(1)}$ and contain only binary values $\{0,1\}$, producing a stark
alternating texture reflecting the deterministic mean-reverting cycle of
chunk~1. Rows $i=7$--$12$ (bottom band) are governed by $W^{(2)}$ and
contain values in $\{0, 0.5, 1\}$, producing a softer texture reflecting
the persistent upward dynamics of chunk~2. A CNN filter traversing the image
detects the texture change at the band boundary and associates it with the
regime transition at $t=6$.}
\label{fig:tmtf-image}
\end{figure}

\textbf{TMTF image.} By \eqref{eq:tmtf}, rows $i \in C_1$ are governed by
$W^{(1)}$ and rows $i \in C_2$ by $W^{(2)}$. The state sequence
$\mathbf{b}^{(1)} = (1,3,2,1,3,2)$ produces a top band whose row textures cycle
through the rows of $W^{(1)}$: all-zero diagonal, strongly off-diagonal. The
state sequence $\mathbf{b}^{(2)} = (1,1,2,2,3,3)$ produces a bottom band whose
textures reflect $W^{(2)}$: heavy diagonal, upper-triangular.
A CNN filter detecting the texture contrast between the two bands identifies a
process that transitioned from mean-reverting to persistent dynamics.
\end{example}

\section{Geometric Interpretation of Local Transition Matrices}

The structure of each local matrix $W^{(k)}$ encodes a direct geometric signature
of the process dynamics within segment $C_k$. We describe the four canonical
patterns corresponding to the main classes of time series behaviour.

\textbf{Diagonal-heavy matrix.} Large diagonal entries $W^{(k)}_{ll}$ for all $l$
indicate the series tends to remain in its current state. The trajectory lingers
in each quantile band. This is the signature of \emph{persistence}: characteristic
of near-unit-root processes ($\phi \lesssim 1$) and slowly drifting series. In
the image, the band produced by chunk $k$ has a dark stripe along the main
diagonal.

\textbf{Spread matrix.} Small diagonal entries and substantial off-diagonal mass
indicate the series departs rapidly from its current state, oscillating across
quantile bands without settling. This is the signature of \emph{mean reversion}:
characteristic of stationary AR(1) processes with $|\phi| \ll 1$. In the image,
the band has a diffuse, near-uniform texture.

\textbf{Upper-triangular matrix.} Probability mass concentrated above the main
diagonal, with negligible lower-triangular entries, indicates the series
predominantly moves to higher quantile states and rarely returns to lower ones.
This is the signature of a \emph{persistent upward trend}. The worked example's
$W^{(2)}$ is of this type. If additionally $W^{(k)}_{QQ} \approx 1$, the top
state is absorbing and the series has approached or crossed an explosive regime
in this segment.

\textbf{Uniform matrix.} All entries approximately equal to $1/Q$ indicates
transitions are nearly equiprobable from any state: the current state carries no
information about the next. This is the signature of a \emph{random walk}, in
which each step is independent of the current level.

These four patterns are the vocabulary with which the TMTF image describes the
dynamics. A CNN trained on TMTF images learns to recognise these textures and
their spatial arrangement within the band structure, associating configurations
of bands---for instance, uniform followed by diagonal-heavy---with specific
generative process histories.

\section{Multi-Resolution TMTF}

A natural extension is to compute multiple TMTF images for different values of
the bin count $Q$, producing a multi-resolution representation. This is directly
analogous to the multi-scale permutation entropy of \cite{kay2024permutation},
who compute permutation entropy over a grid of embedding dimensions $n$ and
delays $\tau$ to capture information content at multiple temporal scales.

For a fixed $K$ and a set of bin counts $\{Q_1, Q_2, \ldots, Q_R\}$, the
multi-resolution TMTF produces $R$ separate $T \times T$ images, each encoding
the transition dynamics at a different quantile resolution. A coarser $Q$ yields
a more robust estimate of gross persistence or mean-reversion patterns; a finer
$Q$ yields more discriminative estimates of the detailed transition structure at
higher variance.

When used as input to a CNN, the $R$ images are stacked as separate input
channels. The convolutional backbone processes all channels simultaneously,
learning complementary features from each resolution. In the stationarity
characterisation application accompanying this paper, $Q \in \{6, 10, 14\}$ are
used with $K=4$ chunks, producing three TMTF channels. Channel saliency analysis
shows that the three channels carry complementary information: the coarser $Q=6$
channel receives the highest gradient attribution for the $\phi$ regression task,
consistent with coarser bins producing more stable per-chunk transition estimates.

The design of multi-resolution TMTF follows the same logic as multi-scale
permutation entropy: both represent a sweep over representation resolutions, and
both produce a family of representations that are individually noisier but
collectively more discriminative than any single member. The TMTF framework
accommodates this extension naturally since each bin count $Q$ yields an
independent $T \times T$ image with its own local transition matrices.

\section{Discussion and Conclusion}

We have introduced the Temporal Markov Transition Field, a representation that
extends the global MTF of \cite{wang2015imaging} to capture time-varying
transition dynamics. The key modification---replacing the single global transition
matrix with $K$ locally estimated matrices---produces a $T \times T$ image with
$K$ horizontal bands, each carrying the transition texture of its temporal
segment.

Three properties make the TMTF well-suited as a CNN input channel for time series
analysis. First, amplitude agnosticism (Proposition~\ref{prop:amplitude}):
invariance to any monotone transformation removes the need for normalisation and
ensures comparability across series at different scales. Second, graceful
degradation (Proposition~\ref{prop:degrade}): when the dynamics are genuinely
stationary, the TMTF reduces exactly to the global MTF, introducing no
unnecessary complexity. Third, interpretability: the texture of each horizontal
band corresponds to a geometrically meaningful property of the local transition
matrix for that segment.

The main design parameter is the chunk count $K$. We have characterised the
bias--variance trade-off: larger $K$ reduces bias under regime-switching but
increases estimation variance per local matrix. The practical guideline
$K \le T/(5Q^2+1)$ ensures sufficient transitions per chunk for reliable
estimation. For series lengths $T \in [200, 1000]$ and bin counts
$Q \in \{6, 10, 14\}$, $K=4$ is a robust default.

The column asymmetry of the TMTF---the chunk function applies to row indices but
not column indices---is a deliberate design choice that preserves cross-time
comparability in the columns while allowing local dynamics to vary by row band.
This means the TMTF image is not decomposable into $K$ independent block matrices:
the off-diagonal blocks carry genuine cross-chunk comparison information.

The relationship between the TMTF and other ordinal representations of time
series---including the multi-scale permutation entropy of \cite{kay2024permutation}
and the multi-scale MTF of \cite{liu2024multiscale}---suggests a family of
representations parametrised by the resolution swept (bins $Q$ in the TMTF,
embedding dimension and delay in permutation entropy) and the temporal
segmentation ($K$ in the TMTF, window size in windowed permutation entropy).
Exploring the theoretical relationships among these representations is a direction
for future work.

\bigskip
\noindent\textbf{Acknowledgements.} This work was conducted as part of the
stationarity characterisation project described in the companion paper.

\bibliographystyle{plain}

\end{document}